\documentclass[letterpaper]{article} 
\usepackage[preprint]{aaai2027}  
\usepackage[hyphens]{url}  
\usepackage{graphicx} 
\usepackage{natbib}  
\usepackage{caption} 
\usepackage{algorithm}
\usepackage{algorithmic}

\usepackage{newfloat}
\usepackage{listings}
\DeclareCaptionStyle{ruled}{labelfont=normalfont,labelsep=colon,strut=off} 
\floatstyle{ruled}
\newfloat{listing}{tb}{lst}{}
\floatname{listing}{Listing}

\usepackage{booktabs}

\usepackage{xcolor}
\usepackage{amssymb}
\usepackage{MnSymbol}
\usepackage{makecell}
\usepackage{float}
\usepackage{array}
\usepackage{placeins}
\usepackage{multirow}
\usepackage{tabularx}
\usepackage{todonotes}
\usepackage{subcaption}

\usepackage{amsthm}
\usepackage{booktabs}

\newcommand{\ARR}[1]{\textcolor{black}{#1}}
\newcommand{\AR}[1]{\textcolor{black}{#1}}

\newcommand{\cut}[1]{}

\newcommand{\Clauses}{\ensuremath{\mathcal{C}}}
\newcommand{\BenClauses}{\ensuremath{\Clauses_{+}}}
\newcommand{\ConClauses}{\ensuremath{\Clauses_{-}}}

\newcommand{\Deals}{\ensuremath{\mathcal{P}}}
\newcommand{\deal}{\ensuremath{P}}

\newcommand{\Users}{\ensuremath{\mathcal{U}}}
\newcommand{\user}{\ensuremath{u}}

\newcommand{\argpaths}{\ensuremath{\mathsf{paths}}}

\newcommand{\Args}{\ensuremath{\mathcal{X}}}
\newcommand{\ArgsP}{\ensuremath{\mathcal{X}_\Deals}}

\newcommand{\ArgsC}{\ensuremath{\mathcal{X}_\Clauses}}
\newcommand{\ArgsCP}{\ensuremath{\mathcal{X}_\Clauses^+}}
\newcommand{\ArgsCN}{\ensuremath{\mathcal{X}_\Clauses^-}}
\newcommand{\ArgsR}{\ensuremath{\mathcal{X}_\mathcal{R}}}
\newcommand{\Atts}{\ensuremath{\mathcal{A}}}
\newcommand{\Supps}{\ensuremath{\mathcal{S}}}
\newcommand{\BS}{\ensuremath{\tau}}

\newcommand{\GS}{\ensuremath{\sigma}}
\newcommand{\QBAF}{\ensuremath{\mathcal{Q}}}
\newcommand{\ZOPA}{\ensuremath{\mathcal{Z}}}

\newtheorem*{example}{Example}
\newtheorem{theorem}{Theorem}
\newtheorem{definition}{Definition}
\newtheorem{proposition}{Proposition}

\newtheorem{corollary}{Corollary}
\newtheorem{lemma}{Lemma}

\newtheorem{dummytheorem}{Theorem}
\newtheorem{dummyproposition}{Proposition}
\newtheorem{dummycorollary}{Corollary}
\newtheorem{dummylemma}{Lemma}

\newboolean{arxivversion}
\setboolean{arxivversion}{true} 

\title{Argumentation for Common Ground: \\ 
Finding Zones of Possible Agreement between Individuals in Conflict}

\author{
   Elisa Cavatorta\equalcontrib and Antonio Rago\equalcontrib\\
}
\affiliations{
    King's College London, UK\\

    \{elisa.cavatorta, antonio.rago\}@kcl.ac.uk
}

\begin{document}

\maketitle

\begin{abstract}

 How can common ground between societies in conflict be identified when citizens' acceptability of peace agreements is shaped by contested narratives? 
    Such acceptability is mediated not only by the clauses that agreements include or exclude, but crucially by citizens' subjective reasoning concerning agreements' clauses. \cut{This reasoning encodes deeper narratives, for example about justice, security and identity.} 
    In this paper, we leverage \cut{the strengths of }computational argumentation \cut{in representing and resolving conflicting reasoning }to introduce a novel approach to identifying mutually acceptable agreements among individuals in conflict, i.e. a Zone of Possible Agreement (ZOPA). First, we 
    introduce a quantitative bipolar argumentation framework tailored to represent each side's 
    reasoning about peace agreements\cut{, which includes both attacking and supporting relations between arguments that encode citizens' narratives about clauses}. We then show how merging these  frameworks can enable negotiators to 
    identify peace agreements that are mutually acceptable. 
    To evaluate our approach under conditions of real-world relevance, we focus on the Palestinian-Israeli conflict, where long-standing policy, practitioner and public interest underscores the demand for methods capable of analysing polarised public reasoning.  
    We show how our framework identifies a ZOPA 
    through theoretical analysis and preliminary experiments using survey data from both existing work and retrieved by a large language model.
    The results illustrate how argumentation can empower negotiators and conflict-resolution teams in mapping feasible ZOPAs grounded in citizens' reasoning. 

\end{abstract}



\section{Introduction}
\label{sec:introduction}

Sustainable peace agreements fundamentally depend on citizens' willingness to comply with new institutions that depart from the status quo. When societies are not ready for compromise, agreements face backlash, rejection, or non-ratification \cite{sher2013}. 
Public acceptability is mediated not only by the provisions agreements include (or exclude), but by narratives about their promised outcomes, subjective arguments about how provisions address grievances, and the risks they are perceived to entail. These subjective arguments are inherently difficult to systematise, yet it is precisely 
therein that the bottlenecks to compromise reside \cite{Loreggia_22}. 

Meanwhile, computational argumentation \cite{Dung_95} is a field within AI which excels in representing knowledge and resolving conflicts therein. The formalisms offered by this rich area of research have been deployed in related settings to that we consider here, e.g. opinion modelling \cite{Tarle_22,Rossie_26} and judgmental forecasting \cite{Irwin_22_KR,Gorska_25}. However, to our knowledge, its technologies have not been deployed in real-world conflict resolution, unlike other areas of AI such as Markov decision processes and linear temporal logic \cite{Kasenberg_18}, or large language models (LLMs) \cite{Tessler_24,Konya_25}.

In this paper, we build on the contributions of \citet{Cavatorta_25}, who examine the acceptability of prospective peace agreements in both Israeli and Palestinian societies and identify the Zone of Possible Agreement (ZOPA), the set of agreements in which two parties can find common ground. Using nationally representative samples and experimentally controlled clause values, the authors estimate causal effects of clause inclusion on agreement endorsement. We extend this work by focusing on the theoretical and empirical evaluation of the reasoning underpinning these population-level parameters, leveraging argumentation to \cut{represent and }analyse individuals' reasoning about agreement clauses. 

To do so, we use \emph{quantitative bipolar argumentation frameworks} (QBAFs) \cite{Amgoud_18,Baroni_18}, i.e. formal argumentation frameworks that represent arguments with an intrinsic strength and positive or negative relations between them. \emph{Gradual semantics}, i.e. quantitative evaluation methods, may then be applied to evaluate an argument's acceptability, which have been shown to be useful in settings from explainable AI \cite{Dejl_21} to online review aggregation \cite{Rago_25}. In doing so, we make theoretical and experimental contributions that justify the use of computational argumentation as a means for supporting tools to assist negotiators and conflict-resolution teams in mapping feasible ZOPAs grounded in citizens' reasoning.
The intended users of these tools are those undertaking analysis of negotiations such as professional mediators, peace negotiators and non-governmental organisations who assess the public viability of specific agreements' clauses. 
We believe that by leveraging the reasoning behind citizens' narratives, our method can help to develop tools for revealing ZOPAs which were previously inaccessible to the designers of acceptable peace agreements.

After giving the necessary preliminaries (§\ref{sec:preliminaries}), we make the following contributions:

\begin{itemize}
    \item We introduce novel \cut{form of }QBAFs tailored to represent conflicting individual's reasoning about \cut{peace }agreements and show theoretically that, if equipped with suitable gradual semantics\cut{ 
    that satisfies certain formal properties}, they intuitively represent individuals' views (§\ref{sec:main}).
    \item We demonstrate how a set of QBAFs representing individual citizens' conflicting reasoning can be merged into a single QBAF to indicate ZOPAs, a subset of the agreements, proving formal guarantees thereon and further restricting the set of suitable gradual semantics
    (§\ref{sec:merging}).
    \item 
    We perform preliminary experiments to evaluate our approach using survey data taken from \cite{Cavatorta_25}, \ARR{in addition to data} from the reports of nationally-representative surveys retrieved using an LLM, illustrating the theoretical results and suitability for real-world deployment with negotiators (§\ref{sec:evaluation}).
\end{itemize}

We then consider the related work in the literature (§\ref{sec:related}), before concluding and looking ahead to future work (§\ref{sec:conclusions}).



\section{Preliminaries}
\label{sec:preliminaries}

\paragraph{Application Context}
Peace agreements are contracts between conflicting parties that aim to resolve the underlying issues causing the conflict. In \cite{Cavatorta_25}, a peace agreement, $\deal \in \Deals$, where $\Deals$ is the set of all peace agreements, is a set of \emph{clauses}, i.e.~proposed changes in, or continuations of, the status quo \ARR{(allowing for one-hot representations of multi-value variables)}. We let $\Clauses = \{ C_1,\ldots, C_n\}$ be a set of $n$ possible clauses representing changes in the status quo, where any $\deal \in \Deals$ is such that $\deal \subseteq \Clauses$ \AR{and $\Deals$ is the power set of $\Clauses$}. 
For example, let us consider a situation in the Israeli-Palestinian context in which we have $\Clauses = \{ C_1, C_2\}$. Here, $C_1$ may be the clause calling for a freeze on settlement building in the West Bank, and $C_2$ may be the clause requiring Palestinians to officially recognise Israel, both changes in the status quo. The absence of such clauses in our peace agreements  represents continuations of the status quo, i.e. continuation of settlement building and no recognition of Israel, resp. Note that the absence of a clause \AR{$C_i$ from a peace agreement $P_j$, i.e. $C_i \nin P_j$,} means 
\AR{that the peace agreement contains the negation of the clause $C_i$, i.e. the continuation of the status quo}.
For example, if we have $\deal_i, \deal_j \in \Deals$, where $\deal_i = \emptyset$ and $\deal_j = \Clauses$, $\deal_i$ represents an agreement with maximal continuation of the status quo (i.e. no change from status quo) and $\deal_j$ represents an agreement with maximal change from the status quo (i.e. all clauses 
\AR{representing} changes in the status quo \AR{are contained in $P_j$}). 
For each clause, a citizen $\user_i$ either endorses it or does not. This partitions the set of clauses such that $\Clauses = \BenClauses^{i} \cup \ConClauses^{i}$ with $\BenClauses^{i} \cap \ConClauses^{i} = \emptyset$, where $\BenClauses^{i}$ is the set of clauses $\user_i$ endorses and $\ConClauses^{i}$ the set they do not. We treat endorsement as binary, so a citizen is never undecided or indifferent about a clause.

\paragraph{Quantitative Argumentation}
\cut{As a formal means for representing and analysing citizens' reasoning, w}
We use the notion of a 
QBAF \cut{\cite{Amgoud_18,Baroni_18}, a quadruple }$\langle \Args, \Atts, \Supps, \BS  \rangle$ where: 
$\Args$ is a finite set of \emph{arguments}
\cut{, i.e. statements which a citizen may reason about, e.g. peace agreements, clauses therein or reasoning thereon}; 
%
$\Atts \!\subseteq\! \Args \!\times\! \Args$ ($\Supps \!\subseteq\! \Args \!\times\! \Args$) is a binary, directed relation of \emph{attack} (support, resp.) between arguments\cut{, representing negative influences on the acceptance of the latter argument;}%
%
\cut{$\Supps \subseteq \Args \times \Args$  is a binary, directed relation of \emph{support} between arguments, representing positive influences on the acceptance of the latter argument;}, where $\Atts$ and $\Supps$ are disjoint; 
$\BS\!:\! \Args \!\rightarrow\! [0,\!1]$ ascribes \emph{base scores} 
to arguments, representing their intrinsic acceptabilities\cut{ of an argument in isolation, i.e. without considering attacks and supports from other arguments}.\footnote{Note that in \cite{Baroni_18}, base scores 
are defined for more general preorders, but here, for simplicity and in line with the majority of existing work, we restrict to $[0,1]$.} 
For any argument $x_i \!\in\! \Args$, we use $\Atts(x_i) \!=\! \{ x_j \!\in\! \Args | (x_j, x_i) \!\in\! \Atts \}$ to denote $x_i$'s set of \emph{attackers}\cut{, i.e. reasons \textit{against} 
$x_i$}, and $\Supps(x_i) \!=\! \{ x_j \!\in\! \Args | (x_j, x_i) \!\in\! \Supps \}$ to denote $x_i$'s set of \emph{supporters}\cut{, i.e. reasons \textit{for} 
$x_i$}.
\cut{The remaining ingredient is a \emph{semantics}: a way to compute how acceptable each argument is once its attackers and supporters have been taken into account. For this w}We deploy gradual semantics, denoted by $\GS$, which, for a given QBAF $\QBAF = \langle \Args, \Atts, \Supps, \BS  \rangle$, 
\cut{take each argument's base score and the strengths of the arguments attacking and supporting it, and }assigns 
each argument $x_i \in \Args$ a \emph{strength} 
$\GS(\QBAF, x_i) \in [0,1]$ \AR{representing its acceptability}. 
In the remainder of this section, we assume as given a generic QBAF $\QBAF = \langle \Args, \Atts, \Supps, \BS \rangle$ with a gradual semantics $\GS$.
Gradual semantics' suitability for specific applications is examined using their theoretical properties. Some of these properties are defined parametrically based on a comparison measure between sets of arguments' strengths. In this paper, we opt for one such measure (without loss of generality) which discounts arguments which zero strength, $\geq_\GS$, as in \cite{Baroni_19}.
Formally, for $X \subseteq \Args$, we define a function returning a multiset by removing zero strength attackers $z(X) = \{\GS(\QBAF,x_i) | x_i \in X, \GS(\QBAF,x_i) \neq 0 \}$.
Then, we use a comparison measure such that for $A,B \subseteq \Args$, we denote: $A =_\GS B$ iff $z(A) = z(B)$; $A \geq_\GS B$ iff there exists an injective mapping $f$ from $z(B)$ to $z(A)$ such that $\forall x_i \in z(B)$, $\GS(\QBAF,f(x_i)) \geq \GS(\QBAF,x_i)$; and $A >_\GS B$, iff $A \geq_\GS B$ and $B \not{\geq_\GS} A$. 
Any $\GS$ satisfies \emph{balance} \cite{Baroni_18} iff $\forall x_i \in \Args$: 
if $\Atts(x_i) =_\GS \Supps(x_i)$ then $\GS(\QBAF,x_i) = \BS(x_i)$;
if $\Atts(x_i) >_\GS \Supps(x_i)$ then $\GS(\QBAF,x_i) \leq \BS(x_i)$; and 
if $\Atts(x_i) <_\GS \Supps(x_i)$ then  $\GS(\QBAF,x_i) \geq \BS(x_i)$.
Any $\GS$ satisfies \emph{monotonicity} \cite{Baroni_18} iff $\forall x_i, x_j \in \Args$: 
if $\BS(x_i) = \BS(x_j)$, $\Atts(x_i) =_\GS \Atts(x_j)$ and $\Supps(x_i) =_\GS \Supps(x_j)$, then $\GS(\QBAF,x_i) = \GS(\QBAF,x_j)$; and 
if $\BS(x_i) \leq \BS(x_j)$, $\Atts(x_i) \geq_\GS \Atts(x_j)$ and $\Supps(x_i) \leq_\GS \Supps(x_j)$, then $\GS(\QBAF,x_i) \leq \GS(\QBAF,x_j)$.
Any $\GS$ satisfies \emph{strict monotonicity} \cite{Baroni_18} iff $\GS$ satisfies monotonicity and $\forall x_i, x_j \in \Args$ such that $\BS(x_i) \leq \BS(x_j)$, $\Atts(x_i) \geq_\GS \Atts(x_j)$ and $\Supps(x_i) \leq_\GS \Supps(x_j)$, and at least one of these relations is strict, then $\GS(\QBAF,x_i) < \GS(\QBAF,x_j)$.
Any $\GS$ satisfies \emph{duality} \cite{Potyka_18} iff $\forall x_i, x_j \in \Args$ such that $\BS(x_i) = 1 - \BS(x_j)$, $\Atts(x_i) = \Supps(x_j)$ and $\Supps(x_i) = \Atts(x_j)$, $\GS(\QBAF,x_i) = 1- \GS(\QBAF,x_j)$.
In this paper, we assess the suitability of two of the most popular gradual semantics, \emph{DF-QuAD} \cite{Rago_16} and \emph{QEM} \cite{Potyka_18}. 
Both of these semantics satisfy monotonicity, balance and duality, but only QEM satisfies strict monotonicity \cite{Baroni_19,Potyka_24}.

\section{Representing Citizens' Reasoning on \\ Peace Agreements with QBAFs}
\label{sec:main}

In this section, we define a framework for representing citizens' reasoning about the acceptability of agreements\cut{ built from sets of clauses (§\ref{ssec:QBAF_def})}, before undertaking theoretical analysis to identify which properties characterise desirable gradual semantics in this setting.\cut{, allowing us to select suitable gradual semantics (§\ref{ssec:QBAF_theory}). For each formal definition, we give an intuition and an example}

Our framework is defined as follows.

\cut{
\subsection{Formal Definitions} 
\label{ssec:QBAF_def}
}




\begin{figure*}[ht]
    \centering
    \includegraphics[width=0.9\linewidth]{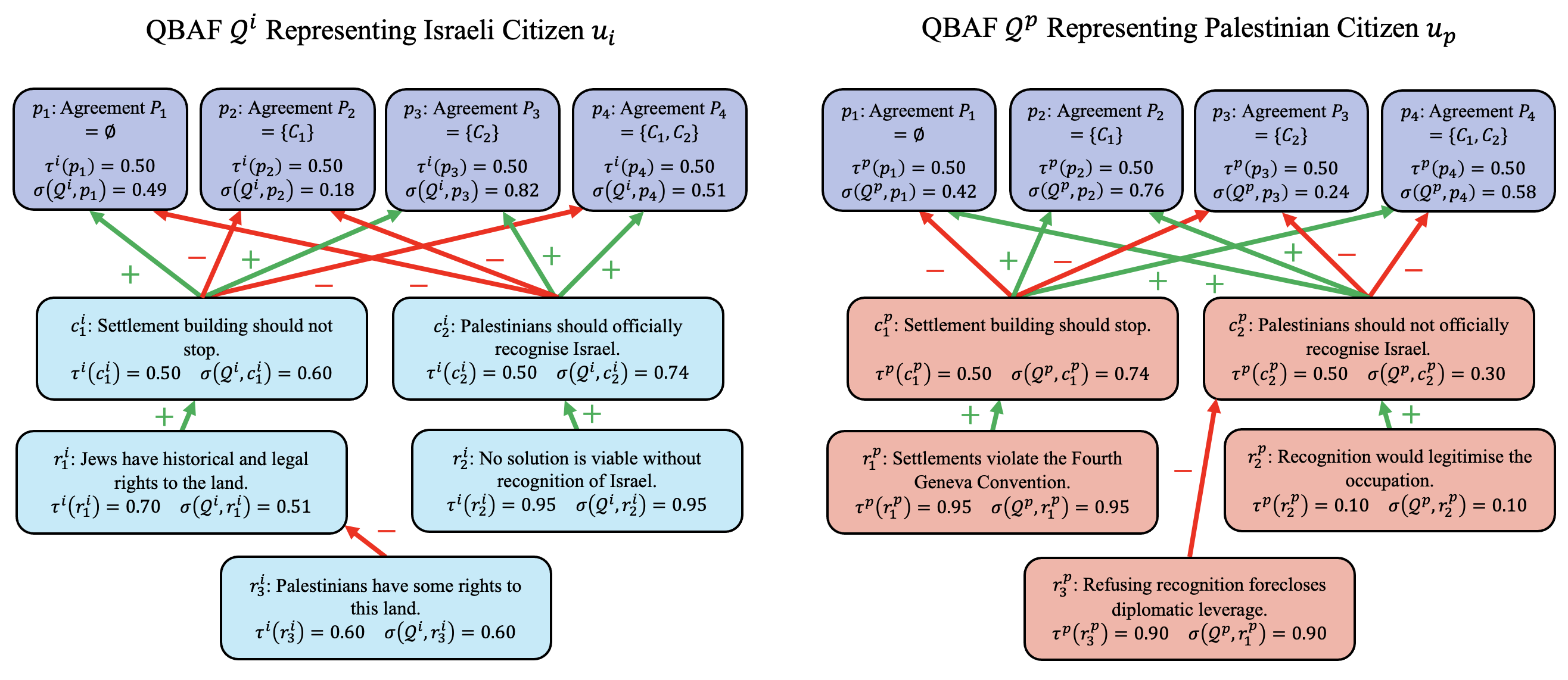}
    \caption{Two QBAFs $\QBAF^i = \langle \Args^i, \Atts^i, \Supps^i, \BS^i \rangle$ (left) and $\QBAF^p  = \langle \Args^p, \Atts^p, \Supps^p, \BS^p \rangle$ (right) representing the reasoning of users $\user_i$ and $\user_p$, resp., on $\Deals = \{ P_1, P_2, P_3, P_4 \}$, in which $\Clauses = \{ C_1, C_2 \}$ where $C_1 =$ \emph{Settlement building stops} and $C_2 =$ \emph{Palestinians officially recognise Israel}, $\BenClauses^i = \ConClauses^j = \{ C_2 \}$, $\ConClauses^i = \BenClauses^j = \{ C_1 \}$, $\ArgsP^i = \ArgsP^p = \{ p_1, p_2, p_3, p_4 \}$, $\ArgsC^i = \{ c_1^i, c_2^i \}$, $\ArgsC^p = \{ c_1^p, c_2^p \}$, $\ArgsR^i = \{ r_1^i, r_2^i, r_3^i \}$ and $\ArgsR^p = \{ r_1^p, r_2^p, r_3^p \}$. 
    Arguments are represented by nodes, attacks by red edges labelled ``$-$'' and supports by green edges labelled ``$+$''.
    \AR{The QEM gradual semantics is used to calculate argument strengths.}}
    \label{fig:framework_example}
\end{figure*}

\begin{definition}\cut{[QBAF representing $\user_i$'s reasoning on $\Deals$]}
\label{def:QBAF}
	Given a set of peace agreements $\Deals$ and a citizen $\user_i \in \Users$, a \emph{QBAF representing $\user_i$'s reasoning on $\Deals$} is a QBAF $\QBAF = \langle \Args, \Atts, \Supps, \BS \rangle$ with gradual semantics $\GS$ such that: 
	\begin{enumerate}
		\item{$\Args = \ArgsP \cup \ArgsC \cup \ArgsR$, where:}
        \begin{itemize}
            \item{$\ArgsP$ is the set of \emph{agreement arguments} where $ |\ArgsP| = |\Deals|$ and $\forall \deal_j \in \Deals$, $\exists p_j \in \ArgsP$;} 
            \item{$\ArgsC$ is the set of \emph{clause arguments}  where $ |\ArgsC| = |\Clauses|$, $\forall C_j \in \Clauses$, $\exists c_j \in \ArgsC$ and $\ArgsC = \ArgsCP \cup \ArgsCN$ such that:}
            \begin{itemize}
                \item $|\ArgsCP| = |\BenClauses|$ and $\forall C_j \in \BenClauses$, $\exists c_j \in \ArgsCP;$
                \item $|\ArgsCN| = |\ConClauses|$ and $\forall C_j \in \ConClauses$, $\exists c_j \in \ArgsCN;$
            \end{itemize}
            \item{$\ArgsR$ is the set of \emph{reasoning arguments};} 
        \end{itemize}
        \item{$\Atts \! \subseteq \! ( \ArgsC \!\times\! \ArgsP) \! \cup\! ( \ArgsR \!\times\! (\ArgsR \!\cup\! \ArgsC))$ and  \\$\Supps \!\subseteq\! ( \ArgsC \!\times\! \ArgsP) \!\cup\! ( \ArgsR \!\times\! (\ArgsR \!\cup\! \ArgsC))$ where:} 
        %
        \begin{itemize}
            \item $\forall c_j \in \ArgsCP$, $\forall p_k \in \ArgsP$, $(c_j,p_k) \in \Atts$ iff $C_j \nin \deal_k$;
            \item $\forall c_j \in \ArgsCP$, $\forall p_k \in \ArgsP$, $(c_j,p_k) \in \Supps$ iff $C_j \in \deal_k$;
            \item $\forall c_j \in \ArgsCN$, $\forall p_k \in \ArgsP$, $(c_j,p_k) \in \Atts$ iff $C_j \in \deal_k$;
            \item $\forall c_j \in \ArgsCN$, $\forall p_k \in \ArgsP$, $(c_j,p_k) \in \Supps$ iff $C_j \nin \deal_k$;
        \end{itemize}
        %
        \item $\BS(p_j) = 0.5$ $\forall p_j \in \ArgsP$.
        %
        %
	\end{enumerate}
\end{definition}

The intuition for each of the points above is as follows.
(1)~Our framework represents the agreements, the clauses and reasoning thereon as arguments. (2) The clause arguments represent $\user_i$'s opinion of whether the corresponding clause \emph{should} or \emph{should not} happen, based on whether the clause is in $\BenClauses^i$ or $\ConClauses^i$, resp. (3) Attacks and supports are such that an argument representing a clause with (without) endorsement from $\user_i$ supports (attacks, resp.) arguments representing agreements that contain the clause, and attacks (supports, resp.) arguments representing agreements that do not contain the clause. 
Meanwhile, (arguments representing\footnote{We may informally refer to agreement/clause/reasoning arguments as agreements/clauses/reasoning, resp., where it is clear we are referring to the QBAF \ARR{and not the entities being represented}.}) reasoning may attack or support clauses or other reasoning. 

Base scores of the agreements are fixed to $0.5$, the midpoint of the range, representing neutral prior acceptance. 
In this paper, we assume that $\BS(c_j) = 0.5$ $\forall c_j \in \ArgsC$, other choices are discussed in §\ref{sec:conclusions} as directions for future work.
In line with other works \cite{Rago_16,Tarle_22}, we limit to acyclic QBAFs in this paper, i.e. not allowing for circular reasoning from citizens, though Definition \ref{def:QBAF} has no such restriction. While we note that $|\Deals|$ is combinatorial in $|\Clauses|$, we take the logical first step of evaluating of all possible agreements, leaving to future work the investigation of algorithms for improved scaling, and noting that strengths in the gradual semantics studied here can be computed in linear time for acyclic graphs \cite{Potyka_19}.

Next, we introduce a ranking over the peace agreements\cut{ that is induced by the gradual semantics in the QBAFs}.

\begin{definition}\cut{[Argumentative Ranking]}
\label{def:ranking}
    Given a set of peace agreements $\Deals$, a citizen $\user_i \in \Users$ and a QBAF representing $\user_i$'s reasoning on $\Deals$, $\QBAF = \langle \Args, \Atts, \Supps, \BS \rangle$ with $\GS$, an \emph{argumentative ranking by $\QBAF$ and $\GS$} is a total ordering over $\Deals$, $\preceq_\GS^\QBAF$, such that $\forall \deal_i, \deal_j \in \Deals$, $\deal_i \simeq_\GS^\QBAF \deal_j$ iff $\GS(\QBAF,p_i) = \GS(\QBAF,p_j)$ and $\deal_i \prec_\GS^\QBAF \deal_j$ iff $\GS(\QBAF,p_i) < \GS(\QBAF,p_j)$. 
\end{definition}

Intuitively, argumentative rankings order the peace agreements based on their strengths, giving a ranking based on their acceptabilities within the QBAF.

Figure \ref{fig:framework_example} gives two examples of QBAFs representing the reasoning of a hypothetical Israeli (left) and a hypothetical Palestinian (right) citizen (superscript $i$ and $p$, resp.). Note that here, $\BenClauses^i = \ConClauses^p = \{ C_2 \}$ and $\ConClauses^i = \BenClauses^p = \{ C_1 \}$, meaning $\user_i$ and $\user_p$ disagree on both of the two clauses, resulting in opposite attack and support relations between the corresponding clause and agreement arguments. The argumentative rankings representing the two citizens' perspectives are $\deal_2 \prec_\GS^{\QBAF^i} \deal_1 \prec_\GS^{\QBAF^i} \deal_4 \prec_\GS^{\QBAF^i} \deal_3$ and $\deal_3 \prec_\GS^{\QBAF^p} \deal_1 \prec_\GS^{\QBAF^p} \deal_4 \prec_\GS^{\QBAF^p} \deal_2$.
Though at first glance, it seems that there is no common ground between the two citizens since they endorse completely different clauses, the argumentative ranking demonstrates that some compromise may be found between the two\cut{ based on the argumentative reasoning}, i.e. agreement $P_4$ in this case. In 
\ARR{the remainder of the paper}, we demonstrate how a set of agreements which is mutually acceptable to both parties can be identified by applying gradual semantics in a principled manner, i.e. ensuring that they satisfy certain properties, and then merging the QBAFs.

We will now assess the behaviour of gradual semantics, as defined by their theoretical properties.
The notation in this section uses a generic QBAF $\QBAF = \langle \Args, \Atts, \Supps, \BS \rangle$ with gradual semantics $\GS$ representing the reasoning of a citizen $\user_i \in \Users$ on agreements $\Deals$. 
When comparing QBAFs for different citizens, we use superscripts: we refer to the $\QBAF$ \ARR{for $\user_i$} as $\QBAF^i = \langle \Args^i, \Atts^i, \Supps^i, \BS^i \rangle$, and to any clause or reasoning argument therein as $c_j^i \in \ArgsC^i$ and $r_k^i \in \ArgsR^i$, resp. Agreement arguments are not assigned superscripts as the same set of agreements is present for all 
citizens.
\ARR{We first consider each of the properties mentioned in §\ref{sec:preliminaries} in turn.}

\textbf{Balance} \ARR{requires that if an argument's attackers are stronger than its supporters, then the argument's strength should be less than or equal to its base score, and vice versa. A violation of balance would create an inconsistency within our setting, e.g. in Figure \ref{fig:framework_example} if $p_2$ were assigned a higher strength than its base score in $\QBAF^i$ when it has stronger opposition than support. We thus believe balanced semantics are essential for intuitive interpretations of citizens' reasoning}.

\textbf{(Strict) Monotonicity} \ARR{requires that increasing the base score, removing/weakening the attackers or adding/strengthening the supporters of an argument can only increase (always increases, resp.) its strength, and vice versa.
These properties thus guarantee an intuitive monotonic relationship between an argument's strength and its attackers, supporters and base score. For example, in Figure \ref{fig:framework_example}, for $\QBAF^p$ we would expect that increasing the strength of $c_2^p$ (i.e. increasing $\user_p$'s negative sentiment towards $C_2$, a clause $P_1$ does not contain) or decreasing the strength of $c_1^p$ (i.e. decreasing $\user_p$'s positive sentiment towards $C_1$, a clause $P_1$ does not contain) could only strengthen $p_1$ (which represents $P_1$). Whether the stronger condition, strict monotonicity, is required, i.e. $p_1$ is always strengthened under these changes, or whether the weaker condition is sufficient, is a question we address in §\ref{sec:merging}}.

\textbf{Duality} \ARR{requires that two arguments which are ``mirror images'' of one another, in terms of its base score, attackers and supporters, should have strengths which are also mirrored about the midpoint $0.5$ of the $[0,1]$ scale.
For example, in Figure \ref{fig:framework_example}, from the Israeli citizen's perspective ($\QBAF^i$), if we take a pair of arguments which have complementary base scores, attackers and supporters, e.g. $p_1$ and $p_4$, it must be the case that $\GS(\QBAF^i,p_1) = 1-\GS(\QBAF^i,p_4)$, given that $p_1$ and $p_4$'s attackers, supporters and base scores are complements of one another. Likewise for $p_2$ and $p_3$, and for the same argument pairs from the Palestinian citizen's perspective ($\QBAF^p$).
We thus require duality because clause endorsement and non-endorsement are constructed as exact mirrors, and we have no principled reason to break that symmetry in the semantics}.

\ARR{Next, we give some theoretical results that further justify monotonicity and duality. First,}
Corollary \ref{cor:P_mirror} shows that the intuitive base score condition holds by default for agreement arguments as their base scores are fixed at $0.5$.\footnote{All proofs are given in the supplementary material.\cut{ It should be noted that alternative set comparison measures (i.e. replacing $\leq_\GS$) may be adopted, which would result in different guarantees potentially tailored to \ARR{more} specific settings.}}

\begin{corollary}
\label{cor:P_mirror}
    For any $p_j, p_k \!\in\! \ArgsP$, if $\Atts(p_j) \!=\! \Supps(p_k)$, $\Supps(p_j) \!=\! \Atts(p_k)$ and $\GS$ satisfies duality, then $\GS(\QBAF,p_j) = 1-\GS(\QBAF,p_k)$.
\end{corollary}

 \cut{We further demonstrate the necessity of monotonicity being satisfied by gradual semantics through two implications, the first of which} 
 \ARR{Our next result} concerns the attackers and supporters of peace agreements, and thus the endorsement of their claims.

\begin{lemma}
\label{lem:clause}
    For any $p_j, p_k \in \ArgsP$, if $\Supps(p_j) \supset \Supps(p_k)$ (and thus $\Atts(p_j) \subset \Atts(p_k)$) and $\GS$ satisfies monotonicity, then $\GS(\QBAF,p_j) \geq \GS(\QBAF,p_k)$. 
\end{lemma}


This result shows that an agreement containing more endorsed clauses and fewer non-endorsed clauses\cut{, i.e. that has more supporters (and thus fewer attackers),} will be 
more acceptable. 
\ARR{For example, in Figure \ref{fig:framework_example}, in $\QBAF^i$ we expect that $\GS(\QBAF^i,p_1) \geq \GS(\QBAF^i,p_2)$, while in $\QBAF^p$ we expect that $\GS(\QBAF^p,p_1) \leq \GS(\QBAF^p,p_2)$.}
\cut{
\begin{example}
In Figure \ref{fig:framework_example}, if $\GS$ satisfies monotonicity, then from the Israeli citizen's perspective ($\QBAF^i$), agreement $p_1$ must be at least as strong as $p_2$, i.e. $\GS(\QBAF^i,p_1) \geq \GS(\QBAF^i,p_2)$, since $p_1$ has more supporters and fewer attackers than $p_2$. Conversely, from the Palestinian perspective ($\QBAF^p$), $p_1$ must 
be no stronger than $p_2$, i.e. $\GS(\QBAF^p,p_1) \leq \GS(\QBAF^p,p_2)$, since $p_1$ 
has more attackers and fewer supporters than $p_1$.
\end{example}
}
\AR{Since citizen's clause endorsement is a fundamental basis of agreement acceptability, gradual semantics' satisfaction of monotonicity seems crucial.}

The next implication concerns the agreement ranking. 


\begin{proposition}
\label{prop:topbottom}
    For $\deal_j\!,\! \deal_k \!\!\in\! \Deals$,  if $\deal_{\!j} \!\!=\! \BenClauses^i$, $\!\deal_k \!\!=\! \ConClauses^i$ and $\GS$ satisfies monotonicity, then $\deal_j \!\!\succeq_\GS^\QBAF\!\! \deal_l$ $\forall\! \deal_l \!\!\in\! \Deals$ and $\deal_k \!\!\preceq_\GS^\QBAF\!\! \deal_m$ $\forall \!\deal_m \!\!\in\! \Deals$.
\end{proposition}

An agreement with total endorsement of its clauses will rank highest \AR{amongst all clauses}, while one with zero endorsement will rank lowest. 
\ARR{In Figure \ref{fig:framework_example}, this means that in $\QBAF^i$ (in $\QBAF^p$) agreement $\deal_3$ is ranked 
highest (lowest, resp.) amongst the agreements given that it has a minimal (maximal, resp.) set of attackers and a maximal (minimal, resp.) set of supporters, which we believe is intuitive behaviour.}

In summary, we have established that balance, monotonicity (though not necessarily strict monotonicity) 
and duality are essential properties for \AR{gradual} semantics in our framework, in that they enforce 
\AR{intuitive} behaviour in 
\AR{the representation of \ARR{individual citizens'} opinions on peace agreements}.
\ARR{Both DF-QuAD and QEM satisfy these requirements, and so would be considered suitable gradual semantics at this point.}

\section{Merging QBAFs to Find ZOPAs}
\label{sec:merging}

This section describes 
how we merge QBAFs representing the opinions of citizens in conflict in order to turn \AR{the} disagreement into a search for common ground. 
We first combine the reasoning of multiple citizens into a single merged QBAF\cut{ and illustrate it with a running example}, defining the ZOPA therein\cut{ (§\ref{ssec:merged_def})}. \cut{We then define the ZOPA, and flag the modelling caveats and possible extensions our minimalist construction leaves open (§\ref{ssec:merged_def}).} We then undertake theoretical analysis \cut{,
establishing formal guarantees which together }to support the choice of gradual semantics\cut{ for 
this application in real-world conflict resolution (\S\ref{ssec:merged_theory})}. 

\cut{
\subsection{Formal Definitions}
\label{ssec:merged_def}
}

We merge citizens' QBAFs as follows.

\begin{definition}\cut{[Merged QBAF]}
\label{def:merged}
    Given a set of $n$ QBAFs $\{ \QBAF^1, \ldots, \QBAF^n\}$, where $\QBAF^i = \langle \Args^i, \Atts^i, \Supps^i, \BS^i \rangle$ for $i \in \{1, \dots, n\}$, with a gradual semantics $\GS$ representing the reasoning of a corresponding set of citizens $U = \{ \user_1, \ldots, \user_n\} \subseteq \Users$ on $\Deals$, the \emph{merged QBAF representing $U$'s reasoning on $\Deals$} is a QBAF $\QBAF^* = \langle \Args^*, \Atts^*, \Supps^*, \BS^* \rangle$ such that:
    \begin{itemize}
        \item $\Args^* = \Args^1 \cup \ldots \cup \Args^n$;
        \item $\Atts^* = (\Atts^1 \cup \ldots \cup \Atts^n) \cap (\Args^* \times \Args^*)$;
        \item $\Supps^* = (\Supps^1 \cup \ldots \cup \Supps^n) \cap (\Args^* \times \Args^*)$;
        \item $\BS^*$ is such that for any $x \in \Args \cap \Args^i$, $\BS^*(x) = \BS^i(x)$.
    \end{itemize}
\end{definition}

\cut{
\begin{corollary}
\label{cor:merged}
    For any merged QBAF $\QBAF^* = \langle \Args^*, \Atts^*, \Supps^*, \BS^* \rangle$ with gradual semantics $\GS$ representing the reasoning of $U = \{ \user_1, \ldots, \user_n \}$ on $\Deals$:
    \begin{itemize}
        \item $\ArgsP^* = \ArgsP^1 = \ldots = \ArgsP^n$;
        \item $\ArgsC^* = \ArgsC^1 \cup \ldots \cup \ArgsC^n$ where $\ArgsC^i \cap \ArgsC^j = \emptyset$ $\forall i, j \in \{ 1, \ldots, n\}$ such that $i \neq j$;
        \item $\ArgsR^* = \ArgsR^1 \cup \ldots \cup \ArgsR^n$ where $\ArgsR^i \cap \ArgsR^j = \emptyset$ $\forall i, j \in \{ 1, \ldots, n\}$ such that $i \neq j$.
    \end{itemize}
\end{corollary}
}

\begin{figure}[t]
    \centering
    \includegraphics[width=0.9\linewidth]{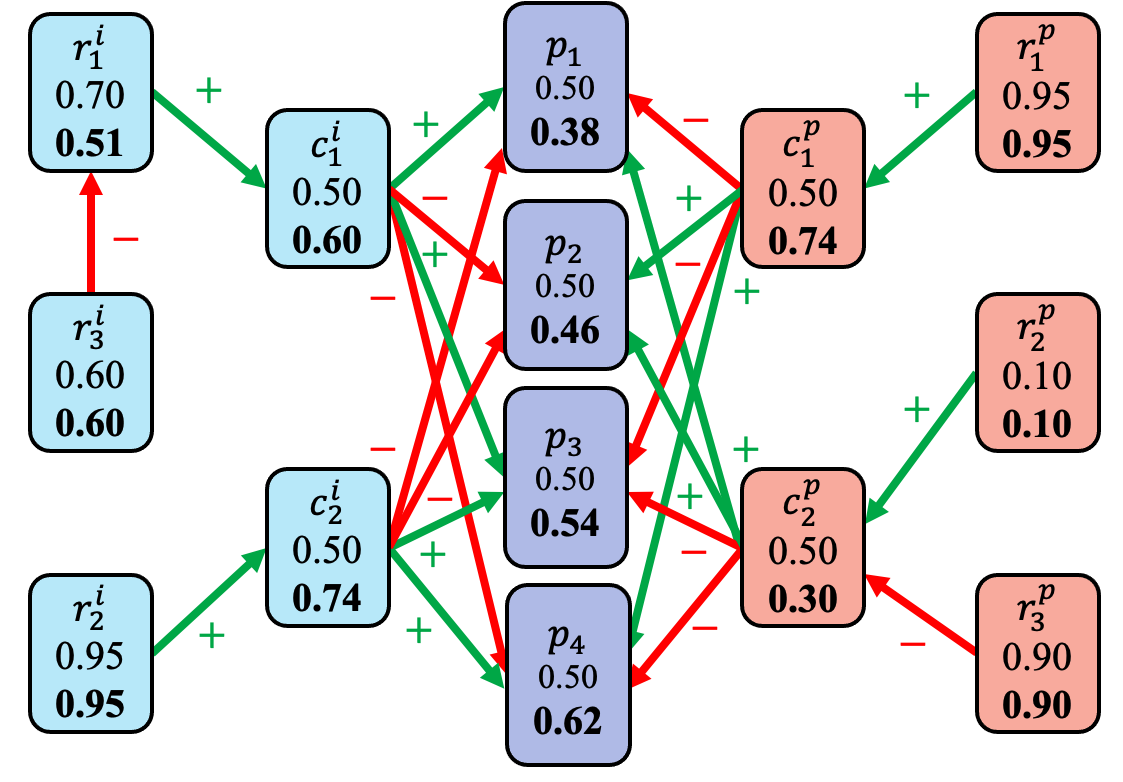}
    \caption{Merged QBAF 
    representing the reasoning of $\{ \user_i, \user_p \}$ on $\Deals$ from Figure \ref{fig:framework_example}, where the values in each argument represent its base score (in normal font) and its strength (in bold font, calculated with \AR{the QEM semantics}).}
    \label{fig:merging_example}
\end{figure}

Intuitively, merging combines several citizens' reasoning into a single graph:  the agreement layer, the set of candidate agreements, is shared across citizens. In contrast, the clause arguments and reasoning arguments are individual-specific and disjoint, and the merging preserves the union of all arguments, along with their corresponding relations and base scores.
For the remainder of this section, we assume as given a generic merged QBAF $\QBAF^* = \langle \Args^*, \Atts^*, \Supps^*, \BS^* \rangle$ with gradual semantics $\GS$ representing the reasoning of $U \subseteq \Users$ on $\Deals$.
With a slight abuse of notation, we allow argumentative rankings \cut{(Definition \ref{def:ranking}) }to be applied to merged QBAFs\cut{ (Definition \ref{def:merged})}.

\cut{Corollary \ref{cor:merged} indicates that this merging ensures that all citizens' reasoning is represented: citizen $u_i$'s endorsement of a clause is a different argument from citizen $u_j$'s non-endorsement of the same clause. Attack and support edges are inherited from each citizen's own QBAF, meaning they are faithfully preserved in the merged QBAF. 
We believe this is a reasonable starting point to represent all citizens' reasoning without further constraints or computation (e.g. identifying and merging individual arguments).}

    Figure \ref{fig:merging_example} illustrates a merged QBAF from the two QBAFs shown in Figure \ref{fig:framework_example}.
    Despite the impression of total disagreement with no common ground between $\user_i$ and $\user_p$ when the QBAFs were viewed individually, the merged QBAF reveals agreements which are mutually acceptable to both citizens based on their own reasoning, i.e. if we take the argumentative ranking for this merged QBAF, $\deal_1 \!\!\prec_\GS^{\QBAF^*}\!\! \deal_2 \!\!\prec_\GS^{\QBAF^*}\!\! \deal_3 \!\!\prec_\GS^{\QBAF^*}\!\! \deal_4$, we see that $\deal_4$ is the most mutually acceptable agreement
    , while $\deal_1$ is the least mutually acceptable.

Next, we introduce a ZOPA, i.e. a classification of which arguments might be considered acceptable by all citizens.

\begin{definition}\cut{[ZOPA]}
\label{def:zopa}
    The \emph{ZOPA} between $U$ in $\QBAF^*$ with $\GS$ is $\ZOPA(\QBAF^*,\GS) = \{ \deal_k \in \Deals \mid \GS(\QBAF^*, p_k) > 0.5 \}$.
\end{definition}

The ZOPA is the set of candidate agreements whose strength in the merged QBAF exceeds the agreements' fixed 
base score of $0.5$\AR{, i.e. the neutral midpoint}. Clearing this threshold means the reasoning from citizens on both sides is on balance supportive of the agreement.

The ZOPA for the example in Figure \ref{fig:merging_example} is $\ZOPA(\QBAF^*\!\!,\GS) = \{ \deal_3, \deal_4 \}$. Agreement $\deal_4$ is the more 
acceptable of the two because, while its clauses received mixed endorsement from the citizens, its supporting reasoning was stronger: 
the attackers of this agreement are the two clauses on which the citizens compromised somewhat in their reasoning ($c_1^i$ for $\user_i$, compromising with the attacker $r_3^i$, and $c_2^p$ for $\user_p$, compromising with the attacker $r_3^p$).
\cut{Meanwhile, $\deal_3$ is in the ZOPA because, while both $\deal_2$ and $\deal_3$ contain clauses which are only endorsed by one side, the reasoning put forward by $\user_i$ (supporting $\deal_3$) is stronger than that put forward by $\user_p$ (supporting $\deal_2$).}
This example demonstrates how reasoning, and the argumentative strength thereof, drives our identification of a ZOPA. The inclusion of $\deal_4$ in the ZOPA demonstrates the importance of compromises in narratives\cut{, while that of $\deal_3$ shows the effect of how strongly participants feel about their reasoning}. Nevertheless, this effect raises the obvious question of the system's susceptibility to strategic manipulation, e.g. if compromises are purposely hidden or strengths are exaggerated, but at this stage, we assume the access to truthful opinions. 

One limitation of our approach is that Definition~\ref{def:merged} takes the disjoint union of clause and reasoning arguments across citizens, meaning \cut{the merged QBAF }\ARR{they }could potentially include duplicate\cut{ argument}s. \cut{The effect of these duplications}\ARR{This} could be addressed 
by merging similar arguments as in \cite{Gorur_26}, which could allow for the extraction new relations between them with argument mining \cite{Cabessa_25,Gorur_25} or the adjustment of base scores based on aggregating citizens' endorsement \cite{Rago_17}. 

Also, our threshold-based ZOPA is one of several ZOPA notions our framework supports natively: it is the most parsimonious choice consistent with balance. The same merged QBAF accommodates threshold-based, top-k and Pareto improvement on the status-quo as direct variants. We leave an investigation of their formal guarantees to future work. 

\cut{
\subsection{Theoretical Analysis}
\label{ssec:merged_theory}
}

We now theoretically analyse our merged QBAF, determining the properties needed in the selected gradual semantics to guarantee intuitive behaviour. Firstly,
our decision to merge the arguments in a simple manner gives the following.

\begin{proposition}
\label{prop:preservation}
    If $\GS$ satisfies balance, then
    $\GS(\QBAF^*, c_i^j) = \GS(\QBAF^j, c_i^j)$ $\forall c_i^j \in \ArgsC^*$.
\end{proposition}

Intuitively, the strengths of clause and reasoning arguments will be preserved in the merged QBAF\AR{, giving provenance to the merged QBAF in that reasoning can be traced back to the citizen from whom it came}. Figure \ref{fig:merging_example} shows why this is the case, with arguments ``upstream'' of the clause and reasoning arguments remaining separate from the others due to the direction of the reasoning. \cut{Note that if a more complicated merging
were undertaken, e.g. merging similar arguments or mining new relations across clauses' sub-graphs, this property may be violated.} 


To capture the dynamics of bilateral disagreements, for the remainder \cut{of the paper }we restrict merged QBAFs to 
\ARR{two citizens, i.e.  $U \!\!=\!\! \{ \user_i, \user_j \}$ (which may represent two homogenous parties)}. \cut{We now give our main theoretical results.}


\begin{theorem}\cut{[Balance of ZOPAs]}
\label{thm:zopa}
    For any $p_k \!\in\! \ArgsP^*$ in $\QBAF^*$ with $\GS$, where
    $\GS$ satisfies balance and strict monotonicity:
        if $\Atts^*\!(p_k) \!<_\GS\! \Supps^*\!(p_k)$, then $\AR{\deal_k} \!\in\! \ZOPA(\QBAF^*\!,\GS)$; and
        if $\Atts^*\!(p_k) \!\geq_\GS\! \Supps^*\!(p_k)$, then $\AR{\deal_k} \!\nin\! \ZOPA(\QBAF^*\!,\GS)$. 
\end{theorem}

An agreement is part of the ZOPA if its attackers in the merged graph are weaker than its supporters. 
\cut{The ZOPA thus consists of those agreements for which the reasoning-weighted balance of objections is outweighed by the reasoning-weighted balance of endorsements, across all citizens at once. Conversely, when objections are not outweighed, the agreement is excluded.} 
Consequently, the ZOPA identifies the common ground in the form of agreements of which the collective reasoning is in support, rather than in opposition.
\ARR{For example, in Figure \ref{fig:merging_example}, the stronger supports from $c_2^i$ and $c_1^p$, compared to the weaker attacks from $c_1^i$ and $c_2^p$, mean $\deal_4$ is in the ZOPA (and ranked highest).}


\cut{
\begin{example}
    In Figure \ref{fig:merging_example}, the stronger supports from $c_2^i$ and $c_1^p$, compared to the weaker attacks from $c_1^i$ and $c_2^p$, means that $\deal_4$ is in the ZOPA (and ranked highest). Meanwhile, the stronger attacks from $c_2^i$ and $c_1^p$, compared to the weaker supports from $c_1^i$ and $c_2^p$, means that $\deal_1$ is not in the ZOPA (and ranked lowest). 
\end{example}
}

We will now assess the argumentative ranking induced by $\QBAF^*$ and $\GS$ in Definition \ref{def:ranking} by considering two extreme cases. \cut{Recall that $\BenClauses^i$ denotes the clauses endorsed by citizen $\user_i$ and $\ConClauses^i$ those they reject.}

\begin{theorem}\cut{[ZOPAs under Total Disagreement]}
\label{thm:merged_disagreement}
    If $\BenClauses^i = \ConClauses^j$, $\ConClauses^i = \BenClauses^j$, $\GS(\QBAF^i,c_k^i) = \GS(\QBAF^j,c_k^j)$ $\forall k \in \{ 1, \ldots, |\Clauses| \}$ and $\GS$ satisfies balance, then $\deal_l \simeq_\GS^{\QBAF^*} \deal_m$ $\forall \deal_l, \deal_m \in \Deals$ and $\ZOPA(\QBAF^*,\GS) = \emptyset$.
\end{theorem}


When two citizens disagree on every clause with reasoning of identical strength, the merged graph is symmetric: every agreement faces exactly as much support as it does attack\cut{. Under balance, no agreement can pull ahead of any other, so the ranking collapses: all arguments are ranked equally} 
and the ZOPA is empty. This shows the framework behaving appropriately in the worst-case: perfectly opposed views with no strength asymmetry results in no common ground.

\cut{
\begin{example}
    The hypothetical citizens in Figure~\ref{fig:merging_example} totally disagree on the clauses. If their
  clause arguments were also of equal strength, all agreements would be ranked equally and the ZOPA would be empty. 
\end{example}
}


\cut{Finally, we consider the other extreme. Recall that $z(\ArgsC)$ is the set of clauses which have non-zero strengths.}

\begin{theorem}\cut{[ZOPAs under Total Agreement]}
\label{thm:merged_agreement}
    If $\deal_k = \BenClauses^i = \BenClauses^j$, $\deal_l = \ConClauses^i = \ConClauses^j$, $|z(\ArgsC)| = |\ArgsC|$ and $\GS$ satisfies balance and strict monotonicity, then $\deal_k \succ_\GS^{\QBAF^*}\!\!~\deal_m$ $\forall \deal_m \in \Deals \setminus \{ \deal_k \}$ and $\deal_l \prec_\GS^{\QBAF^*}\!\!~\deal_n$ $\forall \deal_n \in \Deals \setminus \{ \deal_l \}$. Further, $\deal_k \in \ZOPA(\QBAF^*,\GS)$ and $\deal_l \nin \ZOPA(\QBAF^*,\GS)$.
\end{theorem}


Meanwhile, when both sides of the conflict endorse the same clauses, intuitively, the agreement that includes precisely those clauses is ranked strictly highest and lies in the ZOPA, while the agreement which includes precisely none of those clauses is ranked strictly lowest and lies outside it.

\cut{
\begin{example}
    This effect would be seen in Figure \ref{fig:merging_example} if both citizens had endorsed $C_1, C_2 \in \Clauses$, i.e. meaning $\BenClauses^i = \BenClauses^j = \{ C_1, C_2 \}$, which would then mean that $\deal_4$ would be top-ranked, while $\deal_1$ would be bottom-ranked.
\end{example}
}

In this section, we have demonstrated how 
\AR{both parties' QBAFs} can be merged to reveal ZOPAs between citizens, proving intuitive behaviour can be guaranteed\cut{ concerning the preservation of individual reasoning (Proposition~\ref{prop:preservation}), the relative effect of attackers and supporters on ZOPA membership (Theorem~\ref{thm:zopa}) and ZOPAs under total disagreement (Theorem~\ref{thm:merged_disagreement}) and agreement (Theorem~\ref{thm:merged_agreement})}. These results, in addition to those from §\ref{sec:main}, show that gradual semantics which satisfy the properties of balance, (strict) monotonicity and duality are 
suitable for our application in real-world conflict resolution.
Thus, the QEM semantics is suitable, while DF-QuAD is not given its violation of strict monotonicity.

\section{Empirical Evaluation}\label{sec:evaluation}

We now perform preliminary experiments to assess the suitability of our method for real-world deployment. We do so with survey data from \cite{Cavatorta_25} (§\ref{ssec:survey}) and retrieved data from LLMs (§\ref{ssec:LLMs}). In both experimental settings, we use the same set-up as \citet{Cavatorta_25}, with Israeli respondents on one side and Palestinian respondents on the other, and eight binary clauses forming each agreement. 
The clauses (and the corresponding status quo variant) were: 
1) settlement freeze (or continuation); 
2) recognition of Israel as the nation state of the Jewish people (or lack thereof) ; 
3) establishment of an independent Palestinian state with equitable land swaps (or current jurisdiction); 
4) increased freedom of movement for all people (or current restrictions); 
5) unrestricted rights to access to Holy sites (or current restrictions); 
6) Jerusalem as joint capital (or separate and divided capital cities); 
7) mutual amnesty for prisoners (or current practices of detention); and 
8) proportionality on water rights (or current distribution). \ARR{We use the QEM semantics given the findings from §\ref{sec:main}-§\ref{sec:merging}.}

\subsection{Survey Data from Existing Work}
\label{ssec:survey}

In our first experiment, we assess whether our merged QBAF is able to recover respondent preferences that were measured independently of it.
To do so, we use the data of \citet{Cavatorta_25}, who fielded a 64-agreement, rank-ordering task over the same eight binary clauses with balanced samples of Israelis ($n=1152$) and Palestinians ($n=1152$). We ask whether our merged QBAF\ARR{, with the clause arguments populated using the analysis from \cite{Cavatorta_25},} is able to produce an argumentative ranking which corresponds to a ``ground truth'' empirical ranking that we infer from the raw ranking data by measuring, for each agreement, the share of respondents (from both populations) who rank the agreement above the status quo agreement. 
\ARR{To populate the clause arguments for the Israeli and Palestinian sides, we assign them strengths, assuming the reasoning upstream is implicit since \citet{Cavatorta_25} do not record reasoning. 
However, their empirical design identifies, for each
clause, the proportion $\pi_j^a$, for each party $a \in \{i,p\}$, who prefer an agreement containing clause $j \in \{1,\ldots,8\}$ to the otherwise identical agreement without it. We use those 16 causal
estimates to populate the clauses directly. Party $a$ endorses the change
variant of clause $j$ when \citet{Cavatorta_25}'s coefficient $\beta_j^a>0$. A clause argument $c_j^a$ is given strength
$\sigma(c_j^a,\QBAF^a)=|\pi_j^a-(1-\pi_j^a)|$ with $\pi_j^a=\Lambda(\beta_j^a)$ and
$\Lambda$ the logistic function, so the strength is the margin by which the
endorsed variant wins: if everyone in party $a$ prefers an agreement with clause $j$'s compared to the agreement without clause $j$, the strength is 1; if everyone is indifferent, the strength is 0. 
\cut{For example, take clause B: `recognition of
Israel as the nation state of the Jewish people'. For Israelis
$\beta=+0.373$, so $\pi=0.59$: 59 in 100 prefer an agreement with the clause to the
same agreement without it, the clause is endorsed. For Palestinians $\beta=-0.062$, so $\pi=0.48$, the clause
is opposed.}}

Over the 64 agreements, the argumentative ranking and the empirical ranking correlate at Spearman $\rho=0.448$ ($p=0.000$) and Kendall $\tau=0.293$ ($p=0.000$). Turning the ranking into a set, a pooled majority puts 56 of the 64 deals above the status quo and the merged QBAF puts 63 above the strength of the status quo. This yields a precision of $0.873$ and recall $0.982$. However, because the target is 56/64 it is an easy target. \ARR{While this assessment is by no means perfect, we believe it shows encouraging correlation.}

\subsection{Retrieved Data from LLMs}
\label{ssec:LLMs}

Our second experiment examines whether our method gives intuitive results that would be useful for a peace negotiator.
To assess whether it is feasible we built an LLM-driven pipeline that retrieves reasoning arguments from published survey reports, material a mediator typically possesses, demonstrating how our method could provide information without costly fieldwork. 
The output is a ranking over agreements,
together with the clause-level strengths behind it, which is potentially crucial information for a negotiator.

To elicit public reasoning that reflects contemporary Palestinian and Israeli
views, we prompted an LLM (Opus 4.7 from Anthropic\footnote{\url{https://www.anthropic.com/news/claude-opus-4-7}}\ifthenelse{\boolean{arxivversion}}{}{, see the supplementary material for prompts}) to retrieve reasoning arguments about the eight clauses from the reports of nationally-representative opinion polls conducted by a curated list of well-reputed institutes. 
Every candidate argument was manually validated against the source
document before it was admitted to the QBAFs.\footnote{Technically, we ask for a literal substring
check that the quoted statement appears verbatim in the document, an identity
check on the source name and URL, and a context check that the statement carries
a percentage figure and a token identifying the Israeli or Palestinian
population. Prevalence is recorded as reported and never inverted: 20\% support
for annexation is recorded as 0.2, never as 0.8 opposition.} Each retained
argument is then encoded along two attributes. Its base score is the prevalence
quoted for the relevant population, so a reported 69\% becomes $0.69$. Its
stance is the polarity of its content towards the side's endorsed clause: a reason favouring the clause enters as a
supporter, a reason against it as an attacker. For example, the statement ``69\% of
Palestinians indicated satisfaction with prisoner release'' enters as
a supporter of the corresponding clause with strength $0.69$. \cut{Statements reporting direct clause
endorsement (``X\% support variant $c$'') are not used as reasoning.}
While this approach is token-intensive and relies on human oversight, we believe that it makes good use of public reports and LLMs, grounding arguments in citable polling data, while providing a reasonable preliminary assessment of our approach before it is deployed in the real world.

\begin{figure}[t!]
    \centering
        \includegraphics[width=\linewidth]{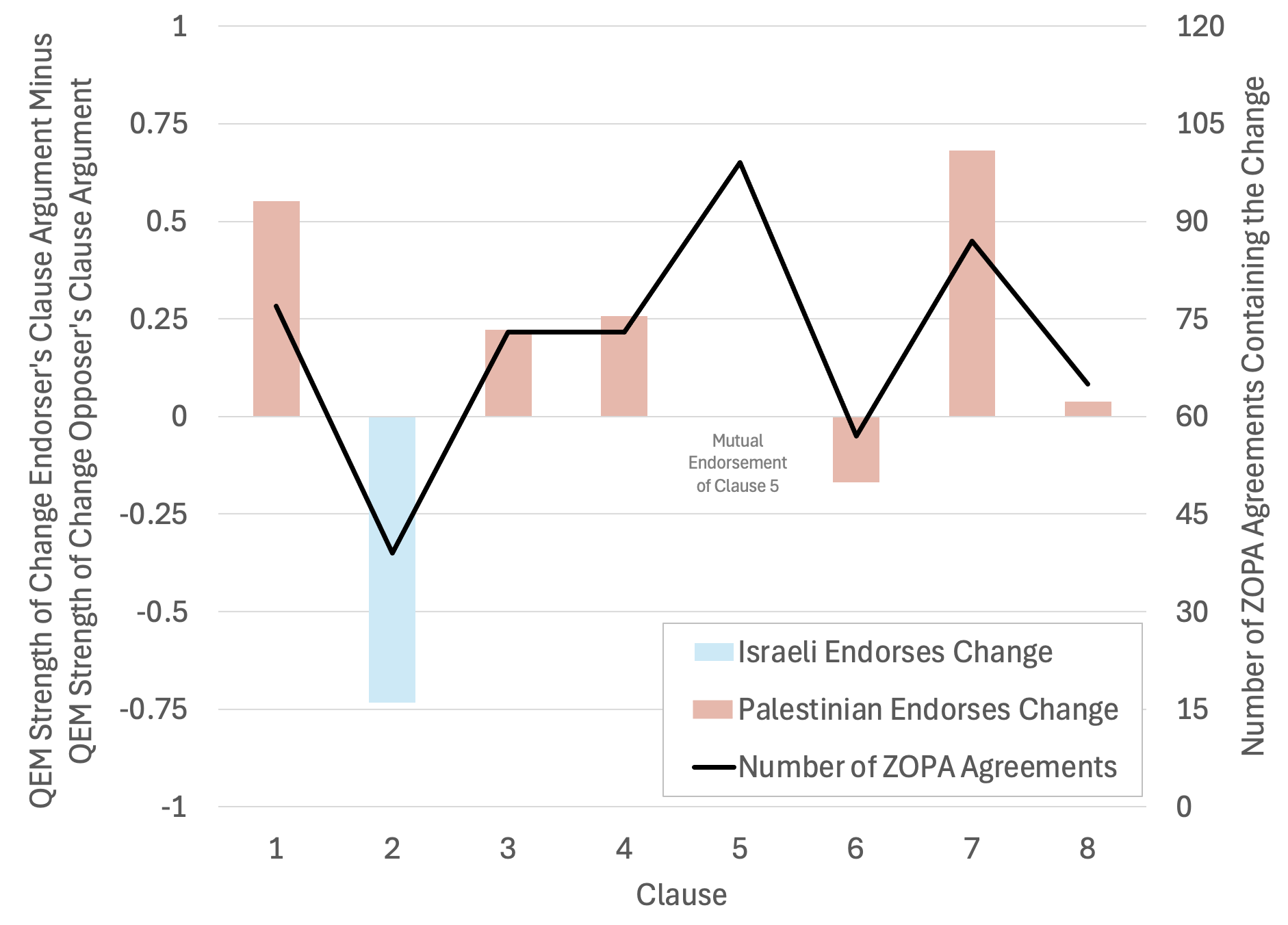}
        \caption{For the data retrieved by the LLM, a comparison, for each clause, of the difference in change endorser versus change opposer clause strength with number of agreements in the ZOPA which contain the clause.}
        \label{fig:llm_chart}
\end{figure}

The approach retrieves 29 validated reasoning arguments, 14 for Israelis (2 in favour of changes, 12 opposing changes) and 15 for Palestinians (11 in favour of changes, 4 opposing changes), drawn from seven distinct opinion-poll report documents, published in 2024-2025. Figure \ref{fig:llm_chart} illustrates the results from this experiment. 
The chart illustrates the potential of our approach in that clauses where the strengths of the clause arguments endorsing a change outweighed those opposing the change (determined by the reasoning) resulted in that change being included in more ZOPA agreements, and vice versa for the opposite case. Particularly encouraging are the facts that: the mutually endorsed change (Clause 5) was in the most ZOPA agreements; the clause with the most positive combined reasoning was in the next most ZOPA agreements (Clause 7) and the clause with the most negative combined reasoning was in the fewest ZOPA agreements (Clause 2). We believe this provides encouraging evidence for the real-world suitability of our approach.

\section{Related Work}
\label{sec:related}

There is a vast body of work on gradual semantics\cut{ for argumentation frameworks}, e.g. \cut{for those which restrict the relations to those}considering only relations of attack \cite{Besnard_01,Leite_11} or support \cite{Amgoud_16}, or those which do not include a base score \cut{on arguments }\cite{Amgoud_08}. 
Those for QBAFs are arguably more popular \cut{are those which include attacks, supports and base scores }\cite{Gonzalez_21,Yun_21,Wang_24}, which potentially align with human reasoning~\cite{Polberg_18,Vesic_22}.
Various analyses of gradual semantics' behaviour have been undertaken~\cite{Mossakowski_18,Delobelle_19,Oren_22,Yu_23,Kampik_24,Anaissy_25}, the findings from which may be useful in our setting, e.g. \cut{extracting }explanations of \cut{the }strengths\cut{ in gradual semantics, potentially giving} for deeper analysis of \cut{respondents' }reasoning. \cut{Open-source implementations for gradual semantics' large scale deployment are also available~\cite{Potyka_22,Alarcon_24}.}
\cut{An outstanding issue in this area is that the strengths computed by many gradual semantics, e.g. DF-QuAD and QEM, do not converge in cyclic graphs due to their recursive nature \cite{Gabbay_15,Anaissy_24}, a limitation which would need to be addressed if we were to allow for cycles in our reasoning arguments.}
\cut{Recently, there have also been a number of studies which have investigated the use of gradual semantics in structured argumentation \cite{Jedwabny_20,Spaans_21,Heyninck_23,Skiba_23,Prakken_24,Rapberger_25,Rago_25_SG}, which allows for more expressivity and depth in the knowledge representation and analysis.}
Gradual semantics' handling of uncertainty and incomplete information has also led to a number of applications in real-world contexts, e.g. fraud detection \cite{Chi_21}, judgmental forecasting \cite{Irwin_22_KR} and various forms of explainable AI \cite{Potyka_21,Potyka_23,Rago_21,Rago_25}. 
\cut{including for neural networks \cite{Potyka_21}, random forests \cite{Potyka_23}, recommender systems \cite{Rago_21} and NLP-driven review aggregation \cite{Rago_25}.} 
To our knowledge, they have not yet been applied to real-world peace agreements. 
Argumentation has also been deployed successfully in negotiation \cite{Kakas_06,Amgoud_11,Bonzon_12,Fossey_25} and automated persuasion \cite{Hadjinikolis_13,Hunter_18,Calegari_20,Donadello_22,Kampik_23}. None of these approaches use gradual argumentation, highlighting the potential of cross-fertilisations with our work. 



\section{Conclusions}
\label{sec:conclusions}

In this paper, we introduced a novel, tailored QBAF for representing citizens' reasoning about peace agreements and showed how merging opposing parties' QBAFs can reveal ZOPAs grounded in evidence-based reasoning. Our theoretical analysis demonstrates that  
  gradual semantics satisfying balance, (strict) monotonicity and duality naturally produce intuitive rankings over agreements. 
\ARR{The empirical evaluation on the Israeli-Palestinian conflict tests the framework against survey data from both existing work and retrieved from an LLM, showing reasonable correlation with the existing data and its suitability for real-world deployment.}
\cut{preferences measured independently of it: with the clause layer populated by the causal
survey estimates of \cite{Cavatorta_25}, the argumentative ranking correlates  with the
shares of respondents who rank an agreement above the status quo. }
\cut{Because the merge is compensatory, it reveals common ground where the parties trade a demand on one clause for a concession on another, yet it is less appropriate for negotiations where a party holds a red line.}
This work shows that argumentation has the potential to assist negotiators in identifying feasible common ground, even amid deeply polarised public discourse.


Our study opens several avenues for future work. One is an empirical evaluation involving reasoning elicited from actual survey respondents or structured interviews processed via NLP\cut{, with the aim of recovering the causal parameters estimated in \cite{Cavatorta_25}}. Scaling the approach to nationally representative samples would require developing efficient methods to merge thousands of individual QBAFs. Methodologically, future work could include developing principled protocols for base score elicitation, 
potentially informed by behavioural principles from behavioural economics. These advances would naturally lend themselves to empirical analyses that pinpoint  which reasoning arguments are the strongest barriers to agreements and identify arguments that, if introduced or reframed, would shift both sides' QBAFs towards mutual acceptability and ultimately support conflict resolution. 


We believe that our contributions highlight the potential of argumentation \emph{in general} in this setting\cut{, and we envisage various extensions of our approach that may be able to capture more complicated dynamics between arguments}. For example, allowing for set-attacks and set-supports\cut{, i.e. attacks and supports which are from multiple arguments towards one particular argument in the spirit of} \cite{Berthold_24} may allow us to model conditional dependencies between clauses. 
Also, adopting a model for approximation of the base scores, \cut{rather than assuming all are equal or asking respondents directly, would be possible. For example,}\ARR{e.g.} preferences (which may be more intuitive to respondents) could instead be elicited from respondents and converted to base scores, as in \cite{Civit_26}. 
Other formalisms, e.g. edge-weighted QBAFs \cite{Yin_26}, probabilistic argumentation \cite{Hunter_17} or structured argumentation \cite{Garcia_04,Toni_14,Modgil_14}, could also provide additional benefits in expressivity.
\cut{We would also leave to future work the investigation of our approach's applicability to other related settings, e.g. contract negotiation, to which (different) argumentative formalisms have already been applied \cite{Dung_08}.} 
Similarly, it would be interesting to \cut{study the use of our setting in}\ARR{assess our method with} downstream tasks, e.g. dynamic opinion polling \cite{Rago_17} or automated persuasion \cite{Donadello_22}.

\newpage

\bibliography{bib}


\newpage
\section*{Supplementary Material}

In this supplementary material, we give additional definitions and the proofs for the theoretical work. \ifthenelse{\boolean{arxivversion}}{}{For the LLM prompts, please see the code and data supplement.}

\subsection*{Additional Definitions}

In the proofs, in order to formalise chains of reasoning from one argument to another via the attack and support relations, for any
$x_i, x_j \in \Args$,
we let a \emph{path} from $x_i$ to $x_j$ be defined as $(x_0,x_1),$ $\ldots,$ $(x_{n-1}, x_{n})$ for some $n>0$
, where $x_0 = x_i$, $x_n = x_j$ and, for any $1 \leq k \leq n$, $(x_{k-1}, x_{k}) \in \Atts \cup \Supps$.
We will use $\argpaths(\QBAF,x_i,x_j)$ to denote the set of all paths between any $x_i, x_j \in \Args$, and we treat paths as sets of pairs. 

The \emph{DF-QuAD semantics} \cite{Rago_16} is a gradual semantics such that for any $x_i \in \Args$, 
    $\GS(\QBAF,x_i) = c(\BS(x_i),\Sigma(\GS(\QBAF, \Atts(x_i))),\Sigma(\GS(\QBAF,\Supps(x_i))))$ 
where, for any $S \subseteq \Args$, $\GS(\QBAF,S)=(\GS(\QBAF,x_1),\ldots,\GS(\QBAF,x_k))$ for $(x_1,\ldots,x_k)$, an arbitrary permutation of $S$, and: 
    %
    $\Sigma$ is such that $\Sigma(())=0$, where $()$ is an empty sequence, and, for $v_1,\ldots,v_n \in [0,1]$ ($n \geq 1$), 
    if $n=1$, then $\Sigma((v_1))=v_1$; if $n=2$, then $\Sigma((v_1,v_2))= v_1 + v_2 - v_1\cdot v_2$; and 
    if $n>2$, then $\Sigma((v_1,\ldots,v_n)) = \Sigma (\Sigma((v_1,\ldots, v_{n-1})),v_n)$; 
    %
    $c$ is such that, for $v^0,v^-,v^+ \in [0,1]$,
    if $v^-\geq v^+$, then $c(v^0,v^-,v^+)=v^0-v^0\cdot| v^+ - v^-|$ and
    if $v^-< v^+$, then $c(v^0,v^-,v^+)=v^0+(1-v^0)\cdot| v^+ - v^-|$.
    %

The \emph{QEM semantics}\footnote{We 
define a simplified gradual semantics here for the case of acyclic graphs.} \cite{Potyka_18} is a gradual semantics such that for any $x_i \in \Args$,
    $\GS(\QBAF,x_i) = \BS(x_i) + (1 - \BS(x_i)) \cdot h(E_{x_i}) -  \BS(x_i) \cdot h(-E_{x_i})$ 
where $E_{x_i} = \sum_{x_j \in \Supps(x_i)}{\GS(\QBAF,x_j)} - \sum_{x_k \in \Atts(x_i)}{\GS(\QBAF,x_k)}$ and for all $v \in \mathbb{R}$, $h(v) = \frac{\max\{v,0\}^2}{1+\max\{v,0\}^2}$.

\subsection*{Proofs}

\begin{dummycorollary}
\label{dumcor:P_mirror}
    For any $p_j, p_k \in \ArgsP$, if $\Atts(p_j) = \Supps(p_k)$, $\Supps(p_j) = \Atts(p_k)$ and $\GS$ satisfies duality, then $\GS(\QBAF,p_j) = 1-\GS(\QBAF,p_k)$.
\end{dummycorollary}

\begin{proof}
    By Definition \ref{def:QBAF}, $\forall p_l \in \ArgsP$, $\BS(p_l) = 0.5$. Then, the proof follows directly from the definition of duality.
\end{proof}

\begin{dummylemma}
\label{dumlem:clause}
    For any $p_j, p_k \in \ArgsP$, if $\Supps(p_j) \supset \Supps(p_k)$ (and thus $\Atts(p_j) \subset \Atts(p_k)$) and $\GS$ satisfies monotonicity, then $\GS(\QBAF,p_j) \geq \GS(\QBAF,p_k)$.  
\end{dummylemma}

\begin{proof}
    By Definition \ref{def:QBAF}, $p_j$ and $p_k$ are such that $\Atts(p_j) \cup \Supps(p_j) = \Atts(p_k) \cup \Supps(p_k) = \ArgsC$ and $\BS(p_j) = \BS(p_k) = 0.5$.
    Then, by monotonicity, $\GS(\QBAF,p_j) \geq \GS(\QBAF,p_k)$.
\end{proof}

\begin{dummyproposition}
\label{dumprop:topbottom}
    For $\deal_j, \deal_k \in \Deals$,  if $\deal_j = \BenClauses^i$, $\deal_k = \ConClauses^i$ and $\GS$ satisfies monotonicity, then $\deal_j \succeq_\GS^\QBAF \deal_l$ $\forall \deal_l \in \Deals$ and $\deal_k \preceq_\GS^\QBAF \deal_m$ $\forall \deal_m \in \Deals$.
\end{dummyproposition}

\begin{proof}
    By Definition \ref{def:QBAF}, $\forall p_n \in \ArgsP$,  $\Atts(p_n) \cup \Supps(p_n) = \ArgsC$ and $\BS(p_n) = 0.5$.
    By the same definition, for any $C_o \in \BenClauses^i$, $c_o \in \Supps(p_j)$ since $C_o \in P_j$ and $c_o \in \Atts(p_k)$ since $C_o \nin P_k$, and conversely for any $C_p \in \ConClauses^i$, $c_p \in \Supps(p_j)$ since $C_p \nin P_j$ and $c_p \in \Atts(p_k)$ since $C_p \in P_k$.  
    Thus, it must be the case that $\Atts(p_j) = \Supps(p_k) = \emptyset$ and $\Supps(p_j) = \Atts(p_k) = \ArgsC$.
    Then, also by Definition \ref{def:QBAF}, any $p_l \in \ArgsP \setminus \{ p_j \}$ is such that $\Supps(p_l) \subset \Supps(p_j)$ and thus $\Atts(p_l) \supset \Atts(p_k)$.
    Similarly, any $p_m \in \ArgsP \setminus \{ p_k \}$ is such that $\Atts(p_m) \subset \Atts(p_k)$ and thus $\Supps(p_m) \supset \Supps(p_k)$.
    Then, by Lemma \ref{lem:clause}, it must be the case that $\GS(\QBAF,p_j) \geq \GS(\QBAF,p_l)$ and $\GS(\QBAF,p_k) \leq \GS(\QBAF,p_m)$, and thus, by Definition \ref{def:ranking}, $\deal_j \succeq_\GS^\QBAF \deal_l$ and $\deal_k \preceq_\GS^\QBAF \deal_m$.
\end{proof}



\begin{dummyproposition}
\label{dumprop:preservation}
    If $\GS$ satisfies balance, then
    $\GS(\QBAF^*, c_i^j) = \GS(\QBAF^j, c_i^j)$ $\forall c_i^j \in \ArgsC^*$.
\end{dummyproposition}

\begin{proof}
    By Definitions \ref{def:QBAF} and \ref{def:merged} it can be seen that $\forall r_k^l \in \ArgsR^*$ such that $ \argpaths(\Args^*,r_k^l,c_i^j) \neq \emptyset$, $l=j$, $\BS^*(r_k^l) = \BS^l(r_k^l)$, $\Atts^*(r_k^l) = \Atts^l(r_k^l)$ and $\Supps^*(r_k^l) = \Supps^l(r_k^l)$.
    By balance, it must then be the case that $\GS(\QBAF^*, r_k^l) = \GS(\QBAF^j, r_k^l)$ and, by the same logic, we can deduce that $\GS(\QBAF^*, c_i^j) = \GS(\QBAF^j, c_i^j)$.
\end{proof}

\begin{dummytheorem}[Balance of ZOPAs]
\label{dumthm:zopa}
    For any $p_k \in \ArgsP^*$ in $\QBAF^*$ with $\GS$, where
    $\GS$ satisfies balance and strict monotonicity:
        if $\Atts^*(p_k) <_\GS \Supps^*(p_k)$, then $\deal_k \in \ZOPA(\QBAF^*,\GS)$; and
        if $\Atts^*(p_k) \geq_\GS \Supps^*(p_k)$, then $\deal_k \nin \ZOPA(\QBAF^*,\GS)$. 
\end{dummytheorem}

\begin{proof}
    Let us first prove that if $\Atts^*(p_k) <_\GS \Supps^*(p_k)$, then $P_k \in \ZOPA(\QBAF^*,\GS)$. 
    Let us compare with some $p_l \in \ArgsP^*$ such that $\Atts^*(p_l) =_\GS \Supps^*(p_l) =_\GS \Atts^*(p_k)$. Here, balance would require that $\GS(\QBAF^*,p_l) = \BS^*(p_l) = 0.5$. Then, for it to hold that $\Atts^*(p_k) <_\GS \Supps^*(p_k)$, it must be the case that $\Supps(x_k) >_\GS \Supps(x_l)$ since $\Atts^*(p_k) = \Atts^*(p_l)$. By Definition \ref{def:QBAF}, $\BS^*(p_k) = \BS^*(p_l)=0.5$ and so, by strict monotonicity, $\GS(\QBAF,x_k) > \GS(\QBAF^*,x_l) = 0.5$.
    Then, by Definition \ref{def:zopa}, $P_k \in \ZOPA(\QBAF^*,\GS)$.
    Next, let us prove that if $\Atts^*(p_k) \geq_\GS \Supps^*(p_k)$, then $P_k \nin \ZOPA(\QBAF^*,\GS)$. Straightforwardly, balance requires that $\GS(\QBAF^*,x_k) \leq \BS^*(x_k) = 0.5$. Then, by Definition \ref{def:zopa}, $P_k \nin \ZOPA(\QBAF^*,\GS)$. 
\end{proof}

\begin{dummytheorem}[ZOPAs under Total Disagreement]
\label{dumthm:merged_disagreement}
    If $\BenClauses^i = \ConClauses^j$, $\ConClauses^i = \BenClauses^j$, $\GS(\QBAF^i,c_k^i) = \GS(\QBAF^j,c_k^j)$ $\forall k \in \{ 1, \ldots, |\Clauses| \}$ and $\GS$ satisfies balance, then $\deal_l \simeq_\GS^{\QBAF^*} \deal_m$ $\forall \deal_l, \deal_m \in \Deals$ and $\ZOPA(\QBAF^*,\GS) = \emptyset$.
\end{dummytheorem}

\begin{proof}
    By Definition \ref{def:QBAF}, $\forall p_n \in \ArgsP^*$,  $\Atts^*(p_n) \cup \Supps^*(p_n) = \ArgsC^*$ and $\BS^*(p_n) = 0.5$.
    By the same definition, for any $C_o \in \BenClauses^i \cap \ConClauses^j$, $c_o^i \in \Supps^*(p_p)$ and $c_o^j \in \Atts^*(p_p)$ for any $P_p \in \Deals$ such that $C_o \in P_p$.
    Meanwhile, $c_o^i \in \Atts^*(p_q)$ and $c_o^j \in \Supps^*(p_q)$ for any $P_q \in \Deals$ such that $C_o \nin P_q$.
    Conversely, for any $C_r \in \ConClauses^i \cap \BenClauses^j$, $c_r^i \in \Atts^*(p_s)$ and $c_r^j \in \Supps^*(p_s)$ for any $P_s \in \Deals$ such that $C_r \in P_s$.
    Meanwhile, $c_r^i \in \Supps^*(p_t)$ and $c_r^j \in \Atts^*(p_t)$ for any $P_t \in \Deals$ such that $C_r \nin P_t$.
    Then, since $\GS(\QBAF^*,c_k^i) = \GS(\QBAF^*,c_k^j)$ $\forall k \in \{ 1, \ldots, |\Clauses| \}$, it must be the case that $\Atts^*(p_l) =_\GS \Supps^*(p_l)$ and $\Atts^*(p_m) =_\GS \Supps^*(p_m)$.
    Balance then requires that $\GS(\QBAF^*,p_l) = \BS^*(p_l) = 0.5$ and $\GS(\QBAF^*,p_m) = \BS^*(p_m) = 0.5$.
    Then, by Definition \ref{def:ranking}, $\deal_l \simeq_\GS^{\QBAF^*} \deal_m$ and, by Definition \ref{def:zopa}, $\ZOPA(\QBAF^*,\GS) = \emptyset$.
\end{proof}

\begin{dummytheorem}[ZOPAs under Total Agreement]
\label{dumthm:merged_agreement}
    If $\deal_k = \BenClauses^i = \BenClauses^j$, $\deal_l = \ConClauses^i = \ConClauses^j$, $|z(\ArgsC)| = |\ArgsC|$ and $\GS$ satisfies balance and strict monotonicity, then $\deal_k \succ_\GS^{\QBAF^*}\!\!~\deal_m$ $\forall \deal_m \in \Deals \setminus \{ \deal_k \}$ and $\deal_l \prec_\GS^{\QBAF^*}\!\!~\deal_n$ $\forall \deal_n \in \Deals \setminus \{ \deal_l \}$. Further, $\deal_k \in \ZOPA(\QBAF^*,\GS)$ and $\deal_l \nin \ZOPA(\QBAF^*,\GS)$.
\end{dummytheorem}

\begin{proof}
    By Definition \ref{def:QBAF}, $\forall p_o \in \ArgsP^*$,  $\Atts^*(p_o) \cup \Supps^*(p_o) = \ArgsC^*$ and $\BS^*(p_o) = 0.5$.
    By the same definition, for any $C_p \in \BenClauses^i \cap \BenClauses^j$, $c_p^i, c_p^j \in \Supps^*(p_k)$ since $C_p \in P_k$ and $c_p^i, c_p^j \in \Atts^*(p_l)$ since $C_p \nin P_l$.
    Conversely, for any $C_q \in \ConClauses^i \cap \ConClauses^j$, $c_q^i, c_q^j \in \Supps^*(p_k)$ since $C_q \nin P_k$ and $c_q^i, c_q^j \in \Atts^*(p_l)$ since $C_q \in P_l$.
    Thus, it must be the case that $\Atts^*(p_k) = \Supps^*(p_l) = \emptyset$ and $\Supps^*(p_k) = \Atts^*(p_l) = \ArgsC^*$.
    Then, also by Definition \ref{def:QBAF}, 
    $\Supps^*(p_m) \subset \Supps^*(p_k)$ and thus $\Atts^*(p_m) \supset \Atts^*(p_k)$.
    Similarly, 
    $\Atts^*(p_n) \subset \Atts^*(p_l)$ and thus $\Supps^*(p_n) \supset \Supps^*(p_l)$.
    Then, by similar logic to Lemma \ref{lem:clause} but taking into account that all clauses have non-zero strength, i.e. $|z(\ArgsC^*)| = |\ArgsC^*|$, it must be the case that $\Atts^*(p_k) <_\GS \Atts^*(p_m)$, $\Supps^*(p_k) >_\GS \Supps^*(p_m)$, $\Atts^*(p_l) >_\GS \Atts^*(p_n)$ and $\Supps^*(p_l) <_\GS \Supps^*(p_n)$. Given that $\BS^*(p_k) = \BS^*(p_m)$ and $\BS^*(p_l) = \BS^*(p_n)$, strict monotonicity requires that $\GS(\QBAF^*,p_k) > \GS(\QBAF^*,p_m)$ and $\GS(\QBAF^*,p_l) < \GS(\QBAF^*,p_n)$. 
    Then, by Definition \ref{def:ranking}, $\deal_k \succ_\GS^{\QBAF^*} \deal_m$ and $\deal_l \prec_\GS^{\QBAF^*} \deal_n$, resp.
    Further, given that $\Atts^*(p_k) = \emptyset <_\GS \Supps^*(p_k)$ and $\Atts^*(p_l) >_\GS \Supps^*(p_l) = \emptyset$, by Theorem \ref{thm:zopa}, $\deal_k \in \ZOPA(\QBAF^*,\GS)$ and $\deal_l \nin \ZOPA(\QBAF^*,\GS)$, resp.
\end{proof}

\end{document}